%% file: paper.tex
\documentclass{article}

\usepackage[final]{neurips_2024}

\usepackage[utf8]{inputenc} %
\usepackage[T1]{fontenc}    %
\usepackage{hyperref}       %
\usepackage{url}            %
\usepackage{booktabs}       %
\usepackage{amsfonts}       %
\usepackage{nicefrac}       %
\usepackage{microtype}      %
\usepackage{xcolor}         %
\usepackage{comment}
\usepackage{amsmath,amsthm,amssymb}
\usepackage{graphicx,enumitem}
\usepackage{cleveref}
\usepackage{caption}
\usepackage{subcaption}

\usepackage[linesnumbered,ruled,vlined]{algorithm2e}

\newtheorem{example}{Example}

\newtheorem{theorem}{Theorem}
\newtheorem{assumption}{Assumption}
\newtheorem{definition}{Definition}
\crefname{assumption}{assumption}{assumptions}
\Crefname{assumption}{Assumption}{Assumptions}
\newtheorem{remark}{Remark}
\renewcommand{\P}{\mathbb{P}}
\newcommand{\E}{\mathbb{E}}
\newcommand{\Cov}{\operatorname{Cov}}
\newcommand{\R}{\mathbb{R}}
\newcommand{\J}{\operatorname{J}}

\DeclareMathOperator{\PR}{PR}

\DeclareMathOperator*{\argmin}{arg\,min}
\DeclareMathOperator*{\DPR}{DPR}
\newcommand{\deq}{\overset{\text{\tiny d}}{=}}

\title{Optimal Classification under Performative Distribution Shift}

\author{%
  Edwige Cyffers\\
  Univ. Lille, Inria, CNRS, Centrale Lille,\\ 
  UMR 9189 - CRIStAL, F-59000 Lille\\
  \texttt{edwige.cyffers@inria.fr} \\
  \And
  Muni Sreenivas Pydi \\
  Université Paris Dauphine, Université PSL,\\
  CNRS, LAMSADE, 75016 Paris\\
  \And
  Jamal Atif\\
  Université Paris Dauphine, Université PSL,\\
  CNRS, LAMSADE, 75016 Paris\\
  \And
  Oliver Cappé \\
  École Normale Supérieure, Université PSL,\\
  CNRS, Inria, DI ENS, 75005 Paris\\
}

\begin{document}

\maketitle

\begin{abstract}
  Performative learning addresses the increasingly pervasive
  situations in which algorithmic decisions may induce changes in the
  data distribution as a consequence of their public deployment.  We
  propose a novel view in which these performative effects are
  modelled as push-forward measures. This general framework
  encompasses existing models and enables novel performative gradient
  estimation methods, leading to more efficient and scalable learning
  strategies. For distribution shifts, unlike previous models which
  require full specification of the data distribution, we only assume
  knowledge of the shift operator that represents the performative
  changes. This approach can also be integrated into various
  change-of-variable-based models, such as VAEs or normalizing flows.
  Focusing on classification with a linear-in-parameters performative
  effect, we prove the convexity of the performative risk under a new
  set of assumptions. Notably, we do not limit the strength of
  performative effects but rather their direction, requiring only that
  classification becomes harder when deploying more accurate
  models. In this case, we also establish a connection with
  adversarially robust classification by reformulating the minimization of the
  performative risk as a min-max variational problem. Finally, we
  illustrate our approach on synthetic and real datasets.
\end{abstract}

\section{Introduction}
\input{intro.tex} %

\section{Push-forward Model for Performative Effects}
\label{sec:pushforward}
\input{pushforward.tex} %

\section{Classification under Performative Shift}
\label{sec:classif}
\input{classif.tex} %

\section{Convexity of Performative Risk}
\label{sec:convexity}
\input{convexity.tex} %

\section{Connection with Robustness and Regularization}
\label{sec:robustness}
\input{robustness.tex} %

\section{Experiments}
\label{sec:expe}
\input{expe.tex} %

\section{Conclusion}
In this work, we have investigated the consequences of assuming a novel,
more explicit, model for performative effects under the form of a push-forward
shift of distribution. We have demonstrated that it comes with practically important
consequences, such as enabling more reliable performative gradient
estimation in large dimensional models.

In the classification case, we
observed that when the change of distribution is given by a linear-in-parameters shift, the performative risk
is convex under relatively general assumptions.
It would be interesting to study how these results may extend to non-linear models for the performative effect.

Finally, we have shown that certain kinds of performative effect induce implicit regularization
of the risk minimization problem. Moreover, this regularization effect can alternatively be
viewed through the lens of adversarial robustness. %
It would be useful to explore whether this reformulation can be used to optimize the performative risk
without an explicit model for the performative effect. 

\section{Acknowledgments}
This work was supported by grants ANR-20-THIA-0014 program
“AI PhD@Lille” and ANR-22-PESN-0014 under the France 2030 program. The authors thank Francis Bach, Bruno Loureiro
and Kamélia Daudel for their helpful insights.

\bibliographystyle{plainnat} 
\bibliography{biblio}

\appendix
\section{Supplementary Material}
\label{sec:app}
\input{supplementary.tex}

\input{appexpe.tex}

\input{checklist.tex}

\end{document}

%% file: intro.tex
Machine learning models are increasingly deployed in real-world
scenarios where their predictions can influence the users' behaviors,
thereby altering the underlying data distribution. This phenomenon,
though rooted in long-standing economic
theory~\citep{morgenstern1928,muth1961rational}, has recently attracted
interest in the machine learning community under the name of {\em
performative prediction}~\citep{perdomo_performative,
hardt2023performative}. Consider for instance a social ranking system:
if it consistently favors a particular subpopulation of individuals,
user behavior might shift towards mimicking the main characteristics
of this subgroup or, conversely, some features of this subpopulation
can undergo modification as a consequence of the selection by the system,
both effects leading to subtle alterations of the original data
distribution. More generally, performative learning captures dynamics
at stake in strategic classification, where individuals are confronted
by algorithmic decisions that impact their life -- such as loan
acceptance, college admission, probation -- and might thus try to
overturn predictions by optimizing some of their features.

This feedback loop, where predictions influence future data, poses new
challenges and necessitates the development of novel approaches within
statistical learning theory and practice~\citep{perdomo_performative,jagadeesan2022regret,drusvyatskiy_stochastic_2020,
hardt2023performative,
zezulka_performativity_2023}.
\cite{perdomo_performative} proposed to formalize performative learning as a 
generalized risk minimization problem, with the \emph{performative risk}
being defined as 
\begin{equation}
    \PR(\theta) = \E_\theta[\ell(Z;\theta)],
\label{eq:pr}
\end{equation}
where $\ell$ is a loss function, $\theta$ a model’s parameters, and
$Z$ an observable random variable drawn from a distribution $\P_\theta$ also
parametrized by $\theta$ itself.  In light of the difficulty of
minimizing $\PR(\theta)$ directly, one can define a \emph{decoupled
performative risk} as $\DPR(\theta, \theta')
= \E_{\theta}[\ell(Z; \theta')]$, clarifying the interplay between the
model’s prediction and the distribution change. This can be seen as a
Stackelberg game that stabilizes when neither the modeler
(learned parameters) nor the environment (distribution) has incentive
to change their states. Solving the performative learning problem
consists in minimizing this risk under the constraint that $\theta
= \theta'$, because the testing samples will follow the distribution
corresponding to the deployed model, and thus $\PR(\theta)
= \DPR(\theta, \theta)$. Minimizing $\DPR(\theta, \theta')$ 
w.r.t. $\theta'$ for a fixed
$\theta$ corresponds to the classical machine learning setting. In contrast,
estimating the performative effect, i.e., knowing how to optimize
$\theta$ for a given $\theta'$ is more challenging as, per definition,
one can only perform statistics from samples collected for values of
the parameters $\theta$ for which the model has already been deployed. Hence,
performative learning does require some form of counterfactual
extrapolation, i.e., what will happen to the data distribution when
the parameter $\theta$ changes from its current setting?

Hence, instead of focusing on methods finding performatively optimal
points, $\theta_{PO} \in \argmin \PR(\theta)$, many previous works,
following \cite{perdomo_performative}, focus on finding stable points
$\theta_{PS} \in \argmin \DPR(\theta_{PS}, \theta)$, through methods
that iteratively minimize the empirical risk. This line of research is
appropriate in settings where the performative effect can be tamed. If
it is sufficiently small, explicitly taking into account the
performative changes of distribution is not required and optimal and
stable points will be close
enough \citep{perdomo_performative}. However, real use cases do not
always satisfy such strong assumptions (see further discussion in
Section~\ref{sec:convexity}). In general, stable points may not be
good proxys for performatively optimal points, particularly in
settings where the performative effect cannot be bounded a priori.

Towards this goal, another line of research focuses on finding the optimal points $\theta_{PO}$. \cite{izzo2021learn} propose to use Monte Carlo sample-based approximations of the gradient of the
performative risk, $\nabla_\theta \PR(\theta)$, based on the score function estimator (see Section~\ref{sec:pushforward}
below). \cite{miller_outside} use a two-stage approach that deploys random models to estimate the performative effect 
in the first stage, and then minimizes the estimated performative risk in the second stage. 
A drawback of both of these approaches is the restrictive set of assumptions needed to show that the algorithms converge.  
While \cite{izzo2021learn} assume the convexity of $\PR(\theta)$ along with smoothness and boundedness contitions,
\cite{miller_outside} assume that the loss function is simultaneously strongly convex and smooth.
Moreover, the score function estimator of \cite{izzo2021learn} necessitates full
knowledge of a parametric form of $\P_\theta$, which is unrealistic in practice.
Alternatively, \cite{jagadeesan2022regret} resort to
derivative-free (or zeroth-order) optimization strategies.
However, such an approach is appropriate only when it is possible to sequentially deploy a large number of model instances, and it does not scale with the dimension of the parameter $\theta$.

The present work is connected to the second line of the research explored above, where the focus is on 
finding the optimal point $\theta_{PO}$. Our contributions are as follows.

{\bf (i) Model the performative effect as a push-forward operator} This novel approach provides a
\emph{new explicit expression of the performative gradient}. Not only does this
approach allow estimation of the performative gradient in settings
where previous methods couldn't, but we show that in typical use cases, the
variance of this new estimator is significantly smaller.

{\bf(ii) Convexity for Performative Classification}
We then focus on the specific task of strategic classification, as
this performative learning problem encompasses various real-use cases
with important societal impact such as college admission or credit
decisions.
Our second contribution is to provide new convexity results on the
performative risk in this case. Whereas existing results were only
proving convexity under assumptions restricting the performative
effect to be small compared to the (assumed) strong convexity of the
loss function $\ell(z;\theta)$, our results leverage structural
assumptions on the performative effect that ensure 
that the performative risk is convex \emph{without any restriction on
the strength of the performative effect}.

{\bf(iii) Linking Performative and Robust Learning} We establish a connexion between performative learning and adversarially robust learning, paving the way to transferring robustness results to
the performative learning field. In particular, this result gives new
insights on the empirical evidence in favor of using regularization in the presence of performative effets. 
Finally, we illustrate our findings on synthetic and real-world datasets.

%% file: pushforward.tex
In this section, we study the general performative learning setting
without yet specializing it to the classification context.
In Section~\ref{subsec: pushforward model}, we introduce the push-forward model of performative learning 
and derive the expression of the gradient of the performative risk under this model.
In Section~\ref{subsec: estimators}, we present a reparameterization-based
estimator for the gradient of the performative risk, and compare it 
to the 
score function based estimator considered by \cite{izzo2021learn}.

\subsection{The Push-forward Model}\label{subsec: pushforward model}

We aim to minimize the performative risk defined in \cref{eq:pr}, where 
the observation $Z$ is drawn from the distribution $\P_\theta$, which depends on the parameter
$\theta\in \R^p$ of the learning model. For this to be tractable, 
one needs additional hypotheses on the nature of the
performative effect. We propose to represent the performative effect through a push-forward
measure, which matches the intuition of having an untouched
distribution that is steered by the performative effect.

\begin{assumption}[Push-forward Performative Model]
    For a given model parameter $\theta \in \R^p$, the samples'
    distribution under the performative effect is given by $\P_\theta
    = \varphi(\cdot;\theta)_\sharp\P$, where $\varphi(\cdot;\theta)$
    is a differentiable invertible mapping on $\mathbb{R}^d$,
    depending on $\theta$.
    \label{assum:pf}
\end{assumption}

This assumption can be equivalently stated by the probabilistic
representation $Z \deq \varphi(U;\theta)$, where $U \sim \P$ (the
symbol $\deq$ denoting equality in distribution).
If $\P$ admits a density, then so does $\P_\theta$ with density function given by 
$p_\theta(z) = |J_z \psi(z;\theta)|
p(\psi(z;\theta))$ where $\psi(\cdot;\theta) =
\varphi^{-1}(\cdot;\theta)$. In this last formula, $J_z
\psi(z;\theta)$ refers to the Jacobian matrix where $[J_z
  \psi(z;\theta)]_{ij} = \frac{\partial \psi(z;\theta)_i}{\partial
  z_j}$.

From an abstract point of view, such a representation of a
parametrized family of distributions exists under very general
conditions. However, the above model is more interesting in scenarios
where $\varphi(\cdot;\theta)$ is a simple operator, for instance a
linear one, as is the case in most of the examples considered by
\cite{perdomo_performative,miller_outside}, and the dependence
with respect to $\theta$ can also be made explicit. This representation is also
modular in the sense that $\varphi$ could be chosen as
the composition $\varphi(u;\theta) = \varphi_0(\varphi_1(u;\theta))$,
where $\varphi_0^{-1}(Z)$ corresponds to a fixed (not depending on
$\theta$) representation of $Z$ in a feature space and
$\varphi_1(\cdot;\theta)$ models the performative effect \emph{in the
representation space}. Although we will not explicitly consider such
cases in the rest of the paper, this representation of the
performative effect is particularly attractive when using embedding tools based on 
kernels~\citep{hofmann2008kernel}, neural nets (e.g., VAEs \citep{kingma2014autoencoding}) 
or normalizing flows
\citep{papamakarios2021normalizing,kobyzev_normalizing}.

This structural assumption on the performative effect yields a new
estimator for the performative gradient, which may be seen as an
instance of the
"reparametrization trick" used in VAEs, normalizing flows or by
\cite{JMLR:v18:16-107}. \cite{mohamed_monte_2020} also refer to this
approach as "pathwise" gradient estimation.

\begin{theorem}[Performative Risk Gradient]\label{thm: perfgrad}
  Under Assumption~\ref{assum:pf},
    the gradient of the performative risk is given by
  \begin{equation}
    \nabla_\theta \operatorname{PR}(\theta) = \E_\theta\left[\nabla_\theta \ell(Z;\theta)\right] + \E_\theta\left[\operatorname{J}_\theta^T \varphi(\psi(Z;\theta);\theta)  \nabla_z \ell(Z;\theta) \right],
    \label{eq:performative_gradient}
  \end{equation}
  where $\nabla_z\ell(z;\theta)$ and $\nabla_{\theta} \ell(z;\theta) $
  denote respectively the gradient with respect to the first and the
  second parameter of the loss, and $\operatorname{J}_\theta^T
  \varphi(u; \theta)$ is the transpose of the Jacobian
  with respect to $\theta$.
\end{theorem}

\begin{proof}
  Notice that under \cref{assum:pf}, we can rewrite the decoupled risk
  with a change of variable as $\DPR(\theta, \theta') =
  \E[\ell(\varphi(U;\theta);\theta')]$. This expression leads to the following.
  \[
        \nabla_\theta \operatorname{PR}(\theta) = \nabla_\theta \E[\ell(\varphi(U;\theta);\theta)] = \E\left[\nabla_\theta \ell(\varphi(U;\theta);\theta) + \J_\theta^T \varphi(U;\theta) \nabla_z \ell(\varphi(U;\theta));\theta)\right],
  \]
  which gives~\cref{eq:performative_gradient} under a change of variable.
\end{proof}
    
\subsection{Estimating the Performative Gradient}\label{subsec: estimators}

From Theorem~\ref{thm: perfgrad}, it is clear that the gradient of performative risk 
in \cref{eq:performative_gradient} is composed  of two
terms -- the first term corresponds to the classical risk minimization, 
while the second one -- which we will refer to as the performative gradient in the following -- captures the performative effect. The first
term can be estimated by $\frac{1}{n}\sum_{i=1}^n \nabla
\ell(Z_i; \theta)$ as usual. For the second term, we propose the following estimator.

\begin{definition}[Reparameterization-based Performative Gradient Estimator]
    The performative gradient $\nabla_{\theta}
    \operatorname{DPR}(\theta,\theta')|_{\theta'=\theta}$ admits as
    unbiased estimator:
    \begin{equation}
      \hat{G}_\theta^{\text{RP}} =
      \frac1n \sum_{i=1}^n \operatorname{J}_\theta^T
      \varphi(\psi(Z_i;\theta);\theta) \nabla_z \ell(Z_i;\theta).
    \end{equation}
\end{definition}

This estimator allows performing gradient descent to minimize the
performative risk and thus, if the performative objective is well
behaved, to converge to the performative optimal point. 
$\hat{G}_\theta^{\text{RP}}$ should be compared to
the following estimator used
by~\cite{izzo2021learn}, which relies on the well-known score function
formula --see \citep{lecuyer1991,kleijnen1996,mohamed_monte_2020} and references
therein. 
\begin{align*}
  \hat{G}_\theta^{\text{SF}} = \frac1n
\sum_{i=1}^n \ell(Z_i;\theta) \nabla_\theta\log p_\theta(Z_i).
\end{align*}
While both $\hat{G}_\theta^{\text{RP}}$ and $\hat{G}_\theta^{\text{SF}}$
estimate the same quantity $\nabla_{\theta}
\DPR(\theta,\theta')|_{\theta'=\theta}$, $\hat{G}_\theta^{\text{RP}}$ has two distinct 
advantages over $\hat{G}_\theta^{\text{SF}}$.
First, computing
$\hat{G}_\theta^{\text{SF}}$ requires  access to the analytical
form of $p_\theta$, which is fairly unrealistic in a learning
scenario, whereas our estimator $\hat{G}_\theta^{\text{RP}}$ only
requires knowledge of $\varphi$, paving the way for a
\emph{semi-parametric approach} in which the performative effect is
modelled explicitly, but not the distribution of the data. For general
maps $\varphi$, $\hat{G}_\theta^{\text{RP}}$ still requires to use
the inverse mapping $\psi$, however this is not required in situations
where the Jacobian $\operatorname{J}_\theta \varphi(u;\theta)$ does
not depend on $u$. Specifically, when $\varphi$ is a shift operator,
one obtains a very simple expression for $\hat{G}_\theta^{\text{RP}}$ 
as shown in the following example.

\begin{example}[Shift Operator]
  \label{ex:shift_operator_gradient}
  If the performative effect can be
      modelled by a shift operator, i.e., $\varphi(U;\theta) = U +
      \Pi(\theta)$, the $\hat{G}_\theta^{\text{RP}}$ estimator is given by:
    \[
        \hat{G}_\theta^{\text{RP}} = \operatorname{J}_\theta^T \Pi(\theta) \, \frac1n \sum_{i=1}^n \nabla_z \ell(Z_i;\theta) ,
    \]
    where $\operatorname{J}_\theta \Pi(\theta)$ is the Jacobian of
   the performative shift $\Pi(\theta)$.
\end{example}

In addition to removing the need to know $p_{\theta}$, a second advantage of $\hat{G}_\theta^{\text{RP}}$
is that it can lead to significant decrease of the variance of the estimates, as illustrated by the following
example.

\begin{example}[Perfomative Gaussian Mean estimation]
  \label{ex:gaussian_mean_estimation}
  Let $\ell(z;\theta) = \|z-\theta\|^2/2$, and $Z
  \deq U + \Pi\theta$, that is, $\Pi(\theta) =
  \Pi\theta$ is a linear shift operator. We will assume $U \sim
  \mathcal{N}(0,\sigma^2 I_d)$, so that $p_\theta(z) \propto \exp[-\|z
    - \Pi\theta\|^2/(2\sigma^2)]$, where $\Pi$ represents the
  performative effect. The gradient of $\operatorname{DPR}(\theta,\theta')$ w.r.t. the distributional
  parameter $\theta$ is given both by
  \begin{align*}
    \nabla_\theta \operatorname{DPR}(\theta,\theta') & = \E_\theta [\Pi^T \nabla_z \ell(Z;\theta')] = \Pi^T \E_\theta[Z - \theta'] = \Pi^T \E[U + a] & \text{(reparameterization)}\\
    & = \E_\theta [\ell(Z;\theta') \nabla_\theta \log p_\theta(z)] = \Pi^T \frac{1}{2\sigma^2} \E\left[\|U+a\|^2 U \right] \qquad & \text{(score function)}
  \end{align*}
  where $a = \Pi\theta - \theta'$. Hence, in this case
  $\hat{G}_\theta^{\text{RP}} = \Pi^T \frac{1}{n}\sum_{i=1}^n
  (U_i+a)$, while $\hat{G}_\theta^{\text{SF}} = \Pi^T \frac{1}{2n
    \sigma^2} \sum_{i=1}^n \|U_i+a\|^2 U_i$. Both of these expressions
  have equal expectation $\Pi^T (\Pi\theta - \theta')$ which
  corresponds to the gradient of $\operatorname{DPR}(\theta,\theta')$
  w.r.t. $\theta$.  However, the reparametrization estimator
  $\hat{G}_\theta^{\text{RP}}$ has covariance $\sigma^2 \Pi^T \Pi /n$
  while the score-based estimator $\hat{G}_\theta^{\text{SF}}$ has
  covariance:
  \[
    \frac1n \Pi^T \left(\frac{(d^2+6d+8)\sigma^2 + 2(d+4)\|a\|^2 + \|a\|^4/\sigma^2}{4}I_d + a a^T \right) \Pi .
  \]
\end{example}

The details of this computation can be found in \cref{app:proofvar}.
Both estimators are unbiased but note that while $\hat{G}_\theta^{\text{RP}}$ would always
be an unbiased estimator of the performative gradient without any further
assumption on the distribution of $U$, the unbiasedness of $\hat{G}_\theta^{SF}$
relies on the fact that $U$ is Gaussian. $\hat{G}_\theta^{RP}$ has a
covariance that does not depend on $\theta$, $\theta'$ nor on the
dimension $d$. In contrast, $\hat{G}_\theta^{SF}$'s covariance
includes a factor that increases with $d^2$, making it unreliable in
high dimensions. It also includes additional terms that grow with the
norm of $\Pi\theta - \theta'$, so the estimator becomes less reliable
when the performative effect is strong.

One could argue that the previous result does not provide a fair
comparison between both estimators, as $G^{SF}$'s variance can be
reduced by subtracting a baseline. Indeed,
$\tilde{G}_\theta^{\text{SF}} = \frac1n \sum_{i=1}^n
(\ell(Z_i;\theta)-m) \nabla_\theta\log p_\theta(Z_i)$ is also an unbiased
estimator of the gradient of the performative effect (for any choice
of the baseline $m$), as the score function has, by definition,
zero expectation. Tuning $m$ properly may reduce the variance by
creating a so-called \emph{control variate} ---see, e.g.,
\cite{greensmith_variance_2004} for the use of this principle in
policy gradient methods. However, similar calculations (detailed in
\cref{app:proofvar}) show that the minimum covariance that
can be achieved by subtracting a baseline is
\[
  \frac1n \Pi^T \left(\left((1+d/2)\sigma^2 +\|a\|^2\right)I_d + a a^T \right) \Pi,
\]
which is still larger than the covariance of
$\hat{G}_\theta^{\text{RP}}$ by a factor that grows with the model dimension $d$.
The fact that the reparameterization-based estimator is preferable when considering the Gaussian
distribution with quadratic loss function was observed before by \cite{mohamed_monte_2020} in the scalar case.
The above computations however show that the difference between the two approaches gets more and more
significant as the dimension increases. The case of other distribution/loss function combination
still needs to be investigated.

%% file: classif.tex
In this section, we specialize to the setting of 
binary classification which encompasses various machine learning
applications where performative effects are expected. Usually, this
setting involves a desirable class and an undesirable one. For
example, the desirable class might represent college admission, loan acceptance, no-spam email, or probation. In this setting, one can also expect that
individuals belonging to the favored class -- we designate this as
class $1$ -- do not need to alter their features, or only with small
changes. On the contrary, individuals with negative predictions -- in
class $0$ -- have an incentive to modify their features, resulting in
a significant performative effect.

We particularize the arguments introduced in 
 \cref{sec:pushforward} to the setting of binary
classification, by fixing $z = (x,y)$, with a covariate vector
$x\in\mathbb{R}^d$ and a label $y\in\{0,1\}$. As is done classically
---see, e.g., \citep{bach2024}, we further assume that the classifier
$f_\theta(x)$ is a real valued function that depends on a parameter
$\theta\in\mathbb{R}^p$ and that a convex loss surrogate $\Phi$ is
used, such that the loss function $\ell(z;\theta)$ is equal to
$\Phi((-1)^y f_\theta(x))$.
We model the performative effect as label-dependent push
forward models, i.e., that, under $\P_\theta$,
$$
  \text{$X|_{Y=1} \deq
  \varphi_1(U_1;\theta)$ and $X|_{Y=0} \deq
  \varphi_0(U_0;\theta)$},
$$
where $\varphi_1$ and $\varphi_0$ represent
the performative changes affecting 
class-conditional distributions of classes $0$ and $1$ respectively. If the classifier $f_\theta(x)$ is
sufficiently expressive, changes such that $\varphi_1 = \varphi_0$
will not create performative  effects. We thus focus on
scenarios where the \emph{performative changes affect 
each class-dependent distribution differently}. For concreteness, we will assume
the following.

\begin{assumption}
\label{assum:fixed_probability}
$\P_\theta(Y=1) = \rho$ is fixed and not subject to performative effects.
\end{assumption}

\begin{assumption}
\label{assum:identity_function}
$\varphi_1$ does not depend on $\theta$, and for simplicity, we assume it is the identity function.
\end{assumption}

\begin{assumption}
\label{assum:drift_operator}
$\varphi_0(u_0; \theta) = u_0 + \Pi(\theta)$ is a shift operator.
\end{assumption}

\Cref{assum:fixed_probability} is a consequence of the
intuitive property that even if the distribution is modified, the
ground truth labels are not impacted by the performative
effect. \Cref{assum:identity_function} allows to focus on the
performative effect on the unfavored class and simplifies the
presentation but could be easily
relaxed. Finally, \cref{assum:drift_operator} is restricting the
performative change to a shift, which does simplify the problem but
still corresponds to a realistic model for feature alteration.

It is important to stress that, despite the fact that this
performative effect is modelled as a shift, the joint distribution of
$Z=(X,Y)$ does not belong to the location-scale family discussed by
\cite{miller_outside}. Under these assumptions, the decoupled
performative risk takes the following form:
\begin{equation}
\operatorname{DPR}(\theta,\theta') = \E_\theta\left[\Phi\left((-1)^Y f_{\theta'}(X)\right)\right] = \rho \E\left[\Phi(f_{\theta'}(U_1))\right] + (1-\rho)\E\left[\Phi(-f_{\theta'}(U_0 + \Pi(\theta)))\right],
\label{eq:DPR_classification_shift}
\end{equation}
where the performative effect is only manifested
in the second term, which corresponds to class 0.

\begin{remark}[Localization of the Performative Shift]
  In \cref{eq:DPR_classification_shift}, we refer to $U_0$ which
  corresponds to the covariates of the second class in the absence of
  performative effect, i.e., when $\theta = 0$. However, for a shift operator,
  for any value of $\bar{\theta}$, one may equivalently write that, under
  $\P_\theta$, $X \deq \varphi(U_{\bar{\theta}};\theta)$,
  where $\varphi(u;\theta) = u + \Pi(\theta)-\Pi(\bar{\theta})$ and
  $U_{\bar{\theta}}$ is distributed under $\P_{\bar{\theta}}$. Thus
  \cref{eq:DPR_classification_shift} can equivalently be rewritten by 
  taking expectation under an arbitrary parameter value $\bar{\theta}$, upon defining
  the performative effect as $U_{\bar{\theta}} + \Pi(\theta)-\Pi(\bar{\theta})$ and the linear model as $f_\theta(x) = x^T(\theta-\bar{\theta})$.
\end{remark}

%% file: convexity.tex
Identification of the cases in which the performative risk $\PR(\theta)$ is convex
is an important step towards generalizing results obtained in
the context of traditional (ie., non performative) learning theory.
Existing results mainly exploit the fact that if the loss function is
strongly convex, a sufficiently small performative effect cannot break
this convexity. For this reason, it is often
assumed \citep{miller_outside, hardt2023performative} that the change
in the distributions has a bounded sensitivity with respect to the
parameters, using the 1-Wasserstein distance:
\[
    W_1\left(P_\theta, P_{\theta'}\right) \leqslant \varepsilon \left\|\theta - \theta'\right\|_2.
\]
Note that in our setting, such $\epsilon$ exists and corresponds to
the operator norm of $\Pi(\theta)$. In order to preserve convexity, it
is then needed that $\epsilon \leqslant \mu/2L$ when the risk is
$L$-smooth and $\mu$-strongly convex. The pricing model, considered by \cite{izzo2021learn}, is a very simple example showing
that the performative risk can be convex while not fulfilling this
criterion.

\begin{example}[Pricing Model]
Given a fixed set of $d$ resources, the pricing model aims at finding
the prices $\theta \in \R^d$ for the $d$ ressources that maximize
the overall profit given the elasticity of the demand level:
\begin{equation}
    \ell(z;\theta) = - z^T\theta \text{ with } Z \deq \varphi(U;\theta) = U - \Pi \theta ,
\end{equation}
where $\Pi$ is a diagonal matrix with positive elements, encoding the
elasticity of the demand level to a raise in price of each resource,
and $\mu=\E[U]$ contains the baseline demand levels for each resource.
\end{example}

In this example, the performative risk $\PR(\theta) = - \sum_{i=1}^d
(\mu_i - \Pi_{ii}
\theta_i)\theta_i$ is a strongly convex quadratic function minimized at $\theta^\star_i = \mu_i/(2\Pi_{ii})$. In contrast, the
decoupled performative risk $\DPR(\theta,\theta') = -(\alpha
- \Pi \theta)^T\theta'$ is still convex, but not strongly convex in
$\theta$ and is always minimized in $\theta'$ at infinity. Despite its
simplicity, this example is thus not covered by existing theorems, and
retraining procedures considered by
\cite{perdomo_performative} fail by diverging.

Moreover, this example highlights that requiring a small sensitivity
for the performative effect does not match the true convexity
conditions. The performative risk is indeed strongly convex as long as
the $\Pi_{ii}$ are positive, irrespectively of their magnitude. In
contrast, in this example, the performative risk would become
non-convex if one of the $\Pi_{ii}$ were negative, even with a small
magnitude.

This motivates the search for related phenomenons in the
classification context, by looking at conditions for ensuring the
convexity of \cref{eq:DPR_classification_shift}. Indeed, we show that
in the classification setting, 
one can also observe convexity without restriction on the magnitude of
the performative effect. For this, we consider the case of linearly
parameterized models, which we denote for simplicity by $f_\theta(x) =
x^T\theta$. Note that, as discussed in Section~\ref{sec:pushforward},
our model of performative effect is composable and hence, we could
also consider the more general linearly parameterized model in which
$f_\theta(x) = \psi(x)^T\theta$, with a linear-in-the-parameter
performative effect in the feature space such that $\psi(X) = U
+ \Pi(\theta)$. For ease of notation, we stick to the case where
$\psi$ is the identity function in the following. Using standard
arguments, the choice of a convex loss surrogate $\Phi$ then entails
that $\operatorname{DPR}(\theta,\theta')$ is a convex function of
$\theta'$. For the same reason, if $\Pi(\theta) = \Pi \theta$ is a
linear-in-parameters shift operator,
$\operatorname{DPR}(\theta,\theta')$ is also convex in $\theta$. Note
however that, unless there is no performative effect (i.e., if
$\Pi=0$), \cref{eq:DPR_classification_shift} is not jointly convex in
$(\theta,\theta')$. The following result show that it is nonetheless
the case that the performative risk $\operatorname{PR}(\theta) =
\operatorname{DPR}(\theta,\theta)$ is convex under the condition that $\Pi$ is a
positive semidefinite matrix (see proof in \cref{app:proofconv}).

\begin{theorem}[Convexity of Classification Performative Risk]
  \label{thm:convex_performative_classif}
  Under \cref{assum:identity_function,assum:fixed_probability,assum:drift_operator},
  for linearly parameterized classifier $f_\theta(x)=x^T\theta$ and
  linear shift operator $\Pi(\theta) = \Pi \theta$, the performative
  risk $\operatorname{PR}(\theta)$ is convex when $\Pi$ is a positive semidefinite
  matrix and one of the following conditions holds.
  \begin{enumerate}[label=(\alph*)] \item $\Phi$ is the
  quadratic loss function; \item $\Phi$ is a convex non increasing
  function (such as hinge, logistic or exp loss).  \end{enumerate}
\end{theorem}

This theorem allows us to extend the known convexity results to losses
that are not strongly convex, and to performative effects with
arbitrary magnitude.

\begin{figure}[hbtp]
    \centering \includegraphics[width=\textwidth]{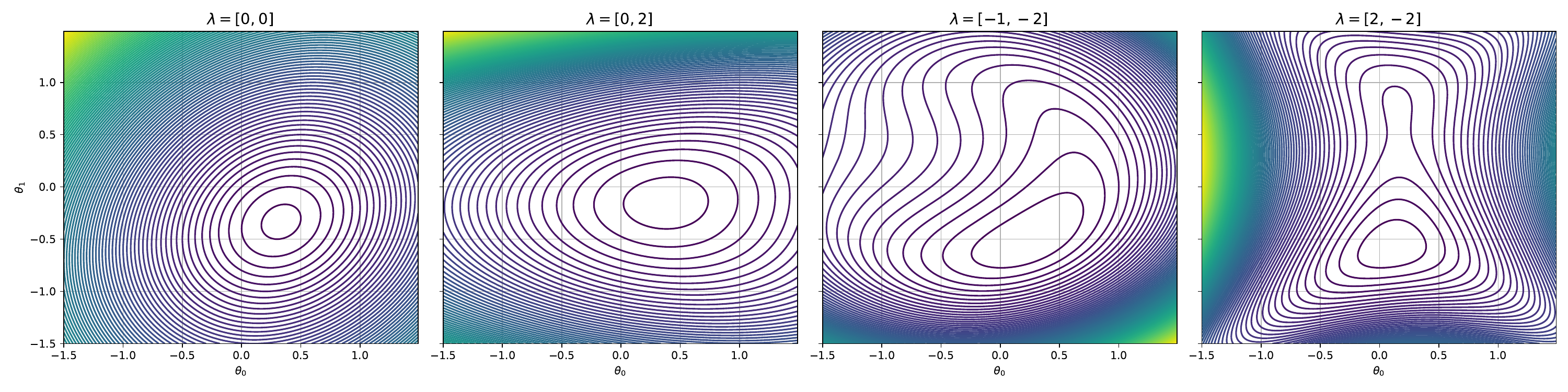}
    \caption{Profile risk for classifying two Gaussian centered in $\mu_0 = (0,0)$ and $\mu_1=(-1, 1)$ with quadratic loss and various values of $\lambda$ for the diagonal coefficients of $\Pi$. The performative risk remains convex as long as $\Pi$ is positive
    semidefinite i.e. $\lambda\geq 0$, and becomes non-convex whenever some of the $\lambda_i$ are negative.}
 \end{figure}

\begin{remark}[Generalization to Performative Effect Affecting Both Classes]
    One could remove \cref{assum:identity_function} to allow class $1$ to change
    under performative effect. The convexity of $\PR(\theta)$ remains
    if $\varphi_1(u;\theta) = u - \Pi_1\theta$, where $\Pi_1$ is  a positive
    semidefinite matrix. Similarly, the classification task becomes harder with performative
    effects and the performative risk is convex.
\end{remark}

%% file: robustness.tex
In order to enforce strong convexity of the loss function $\ell(\cdot;\theta)$, previous
works on performative prediction have considered the use of an additional regularization
term ---see, e.g., Section 5.2 of \citep{perdomo_performative} where logistic
regression is used with a ridge regularizer. When doing so, it has been observed empirically
that the retraining method performs quite well. To build on this observation, we show below
that for linear-in-the-parameter performative effects that tends to make the
classification task harder, the performative optimum may indeed be interpreted as
a regularized version of the base classification problem. This regularization does
not take the form of and additive penalty but can be interpreted as the solution of a specific
adversarially robust classification objective. In this section, we use the slightly stronger assumption that $\Pi$ is
a symmetric positive definite matrix, in order to ensure
that both $\|v\|_\Pi = (v^T\Pi v)^{1/2}$ and $\|v\|_{\Pi^{-1}} = (v^T\Pi^{-1} v)^{1/2}$
are norms on $\mathbb{R}^d$.

\begin{theorem}[Variational Formulation of the performative Risk]
    \label{thm:variational_form_performative_risk}
    Under \cref{assum:identity_function,assum:fixed_probability,assum:drift_operator},
    for linearly parameterized classifiers $f_\theta(x)=x^T\theta$ and
    linear shift operators $\Pi(\theta) = \Pi \theta$, and assuming that $\Phi$ is a
    convex non increasing function and that $\Pi$ is symmetric positive definite, the
    performative risk may be rewritten as
    \begin{equation} \label{eq:adversarial_formulation} \operatorname{PR}(\theta)
    = \rho \E[\Phi(U_1^T\theta)] + (1-\rho)\E\left[\max_{\{\Delta U_0
    : \|\Delta
    U_0\|_{\Pi^{-1}} \leq \|\theta\|_{\Pi}\}}\Phi(-(U_0+\Delta
    U_0)^T\theta)\right] .
  \end{equation}
  \label{thm:robust}
\end{theorem}

Intuitively, for a classification-calibrated loss function
\citep{bartlett_risk_2006,bach2024} and classes with identical
covariances, we expect $\theta$ to align with the direction of
$\mu_1(\theta)-\mu_0(\theta)$, so that, when $\Pi$ is positive definite, the
performative shift $\Pi\theta$ has itself a positive dot product with
$\mu_1(\theta)-\mu_0(\theta)$. The reformulation of the
performative risk in~\cref{eq:adversarial_formulation} formalizes
this intuition by showing that the performative optimum is associated
to an adversarially robust classification task
\citep{GoodfellowSS14,madry2018towards,ribeiro2023regularization} in which the
points of class $0$ are allowed to shift towards those of class
$1$, so as to increase the overall loss. Compared to objectives found in the robust
classification literature, the specificity of~\cref{eq:adversarial_formulation} lies
in the fact that the tolerance (or budget) on the adversarial displacement $\Delta U_0$
depends on both $\Pi$ and $\theta$.

To understand the role played by the $\|\cdot\|_\Pi$ and $\|\cdot\|_{\Pi^{-1}}$
norms, consider the particular case where only a subset of the variables have a
performative effect, i.e., if we let $\theta$ and $U_0$ be partitioned
into
\[
  \theta= \begin{pmatrix} \theta_p \\ \dots \\ \theta_s\end{pmatrix} \, \text{and } \, U_0 = \begin{pmatrix} U_{0,p} \\ \dots \\ U_{0,s}\end{pmatrix} ,
\]
with $\Pi_p = \gamma I$ and $\Pi_s = \epsilon I$, one obtains,
letting $\epsilon$ tend to zero, that the performative risk is equal to
\[
    \operatorname{PR}(\theta) = \rho \E[\Phi(U_1^T\theta)] + (1-\rho)\E\left[\max_{\{\Delta U_{0,p} : \|\Delta U_{0,p}\| \leq \gamma \|\theta_p\|\}}\Phi\left(-(U_{0,s}^T\theta_s + (U_{0,p}+\Delta U_{0,p})^T\theta_p)\right)\right] .
\]
The above expression shows that in this case, only the coordinates subject to the
performative effect appear in the adversarial reformulation. 

In the proof of Theorem~\ref{thm:robust} (see \cref{app:proofrob}), we
observe that the second term of~\cref{eq:adversarial_formulation} may
also be rewritten as $(1-\rho)\E[\Phi(-U_0^T\theta
- \|\theta\|^2_\Pi)]$. Similarly to the case studied by \cite{ribeiro2023regularization},
the term $\|\theta\|^2_\Pi$ that appears inside the surrogate loss function $\Phi$
has a regularization effect. Note however that it is not equivalent to the use of a standard ridge
regression penalty on $\theta$. The following theorem provides a bound on the performative
optimum that highlights the role played by $\Pi$ on the significance of this
regularization effect. 

\begin{theorem}[Regularization Bound]\label{thm:regularization}
  Define $\mu_i = \E[U_i]$. Under \cref{assum:identity_function,assum:fixed_probability,assum:drift_operator},
  for linearly parameterized classifiers $f_\theta(x)=x^T\theta$ and
  linear drift operators $\Pi(\theta) = \Pi \theta$, when $\Phi$ is a
  convex non increasing function and $\Pi$ a symmetric positive definite matrix,
  the minimizer $\theta^*$ of $\operatorname{PR}(\theta)$ satisfies the following condition.
  \begin{align}
    \|\theta^*\|_\Pi  \leq \frac{\| \Pi^{-\frac{1}{2}} (\rho \mu_1 - (1-\rho)\mu_0) \|}{1-\rho} .
    \label{eq:regularization_bound}
  \end{align}
\end{theorem}

Theorem~\ref{thm:regularization} shows that the performative optimum has a smaller
value in $\|\cdot\|_\Pi$ norm when the performative effect is stronger,
that is, when $\Pi$ gets larger. In the particular case where $\Pi = \gamma I$, \cref{eq:regularization_bound} rewrites as $\gamma^{1/2} \|\theta^*\| \leq \gamma^{-1/2} \|\rho \mu_1 - (1-\rho)\mu_0 \|/(1-\rho)$ and thus larger values of $\gamma$ decrease the r.h.s. while the l.h.s. increases for identical values of $\|\theta^*\|$, showing that $\|\theta^*\|$ has to decrease to zero.

%% file: expe.tex
In this section\footnote{The code is available at \url{https://github.com/totilas/PerfOpti}}, we test the performance of our algorithm
Reparametrization-based Performative Gradient (RPPerfGD) with respect
to existing algorithms. Three baselines were introduced
in \cite{perdomo_performative}. First, \emph{Repeated Risk
Minimization (RRM)} computes at each step the next $\theta$ to
minimize the non-performative risk, leading to the update rule
$\theta^{t+1} = \argmin_{\theta'} \DPR(\theta^t, \theta')$. In
practice, we found that this algorithm is unstable as soon as
performative effects become significant. We thus report separately the results
obtained with this algorithm in \cref{app:rrm}. A second baseline
is \emph{Repeated Gradient Descent (RGD)}, which ignores the
performative effect but limits itself to a gradient step towards this
minimization $\theta^{t+1} = \theta^t
- \eta \left.\nabla_{\theta^{\prime}} \DPR\left(\theta, \theta' \right)\right|_{\theta^{\prime}=\theta}$. In
numerical experiments, it is often chosen to add a regularizer to the
objective function, which is particularly interesting in the context
of performative learning, as discussed
in \cref{sec:robustness}. Hence, we report \emph{Regularized Repeated
Gradient Descent (RRGD)}, that corresponds to the repeated gradient
with a loss function including an additional ridge penalty on $\theta$, leading to a more
conservative behavior that is also more robust to performative effects. Finally, we
also compare to the \emph{Score Function Performative Gradient Descent
(SFPerfGD)} which estimates the performative part of the gradient
using the $\hat{G}_\theta^{\text{RP}}$ estimator based on the
score-function approach (see \cref{sec:pushforward}).
This was previously used in small dimensions in \cite{miller_outside}.

\begin{figure}[htb]
    \centering
    \begin{subfigure}[b]{0.31\textwidth}
        \centering
        \includegraphics[trim={0 0 15cm 0},clip,width=\textwidth]{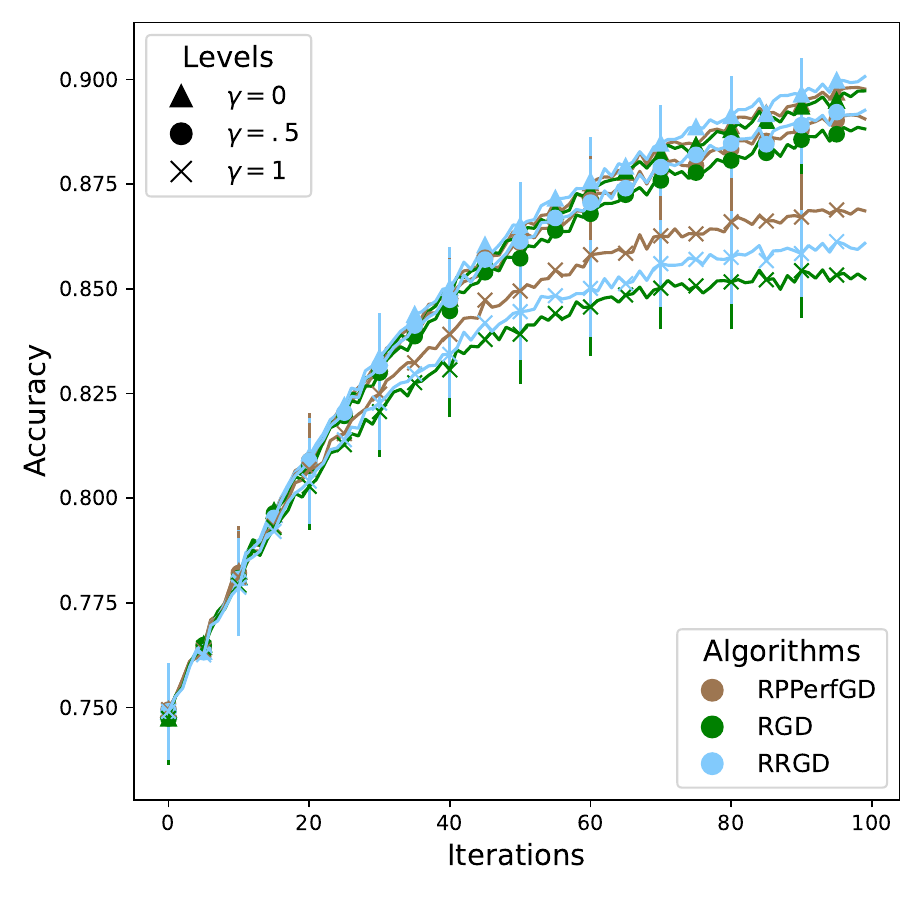}
        \caption{}
        \label{fig:log2dtraj}
    \end{subfigure}
    \hfill
    \begin{subfigure}[b]{0.3\textwidth}
        \centering
        \includegraphics[trim={15.6cm 0 0 0},clip,width=\textwidth]{log2dtraj2.pdf}
        \caption{}
        \label{fig:log2dtraj}
    \end{subfigure}
    \hfill
    \begin{subfigure}[b]{0.3\textwidth}
        \centering
        \includegraphics[width=\textwidth]{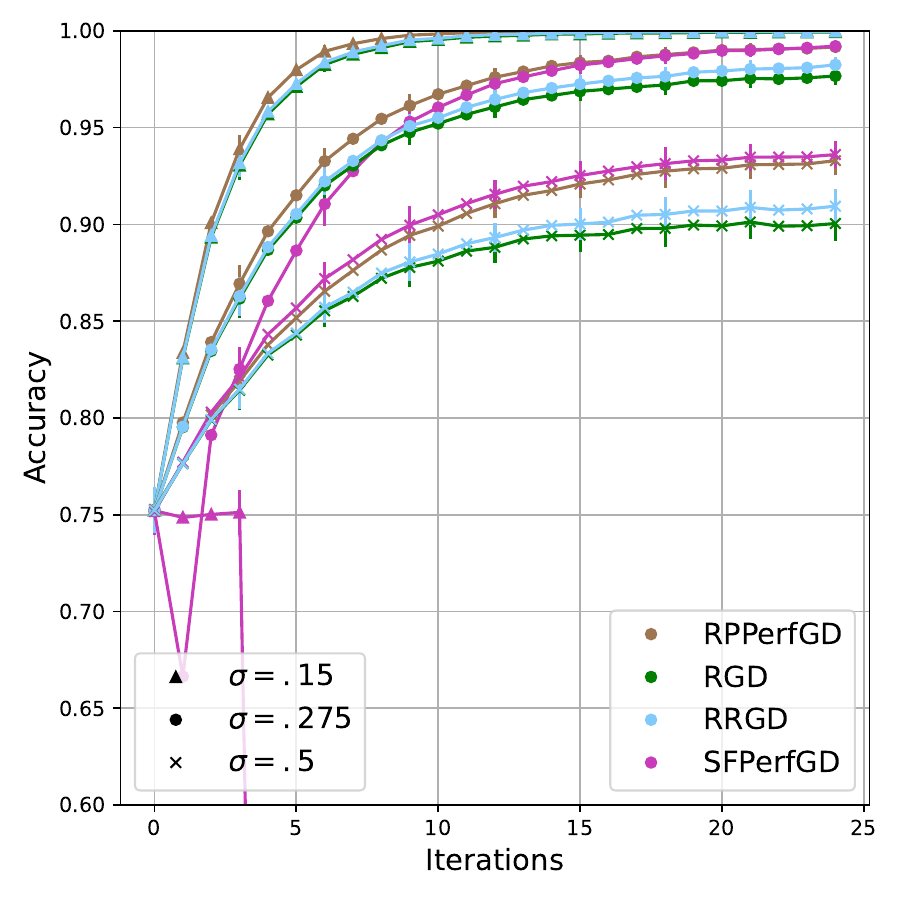}
        \caption{}
        \label{fig:sfvsrp}
    \end{subfigure}
    \hfill
    \begin{subfigure}[b]{0.31\textwidth}
        \centering
        \includegraphics[width=\textwidth]{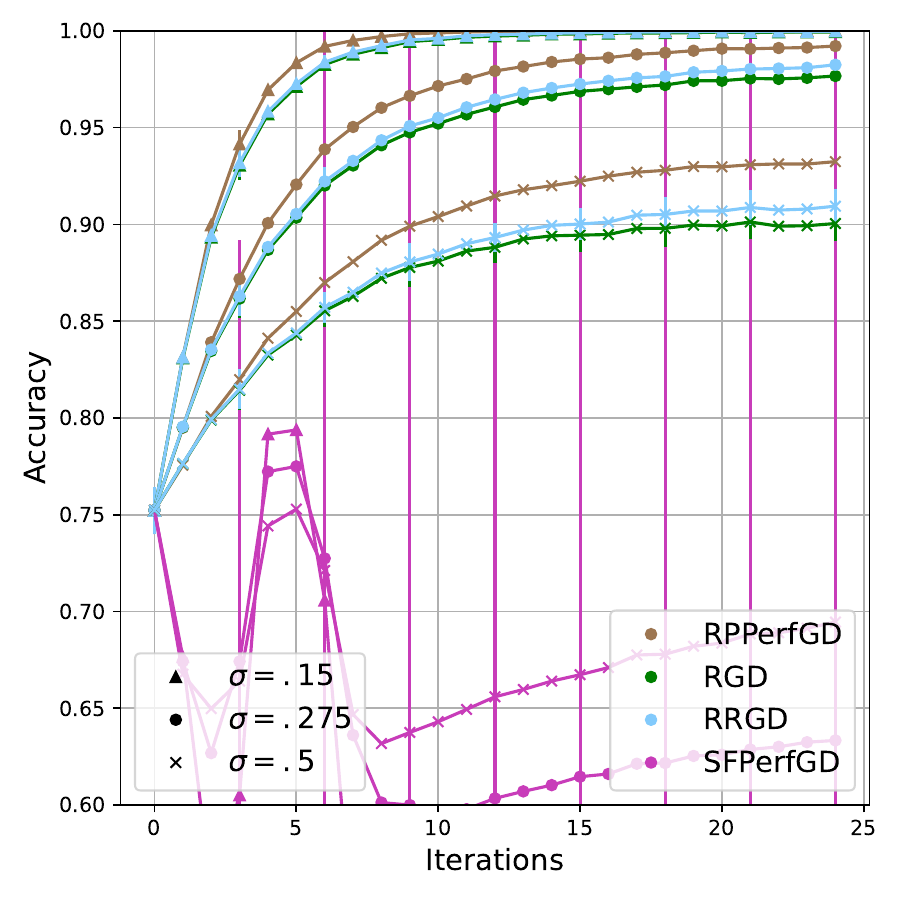}
        \caption{}
        \label{fig:withlearning}
    \end{subfigure}
    \hfill
    \begin{subfigure}[b]{0.31\textwidth}
        \centering
        \includegraphics[width=\textwidth]{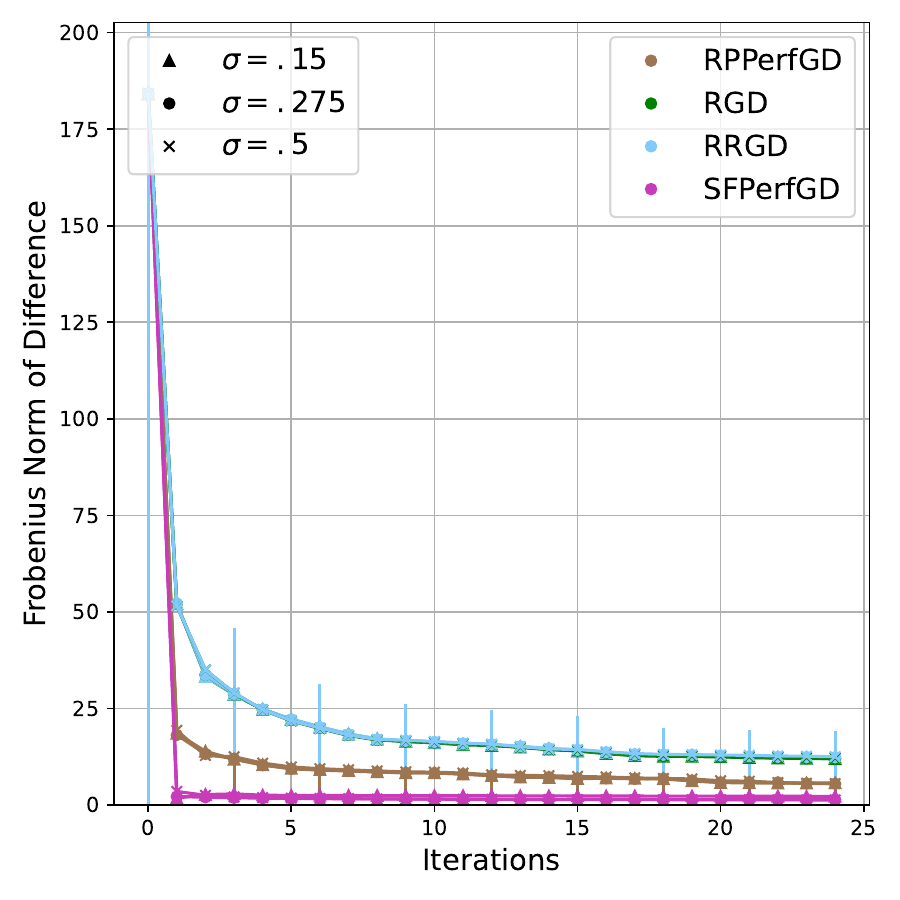}
        \caption{}
        \label{fig:norm}
    \end{subfigure}
    \hfill
    \begin{subfigure}[b]{0.31\textwidth}
        \centering
        \includegraphics[width=\textwidth]{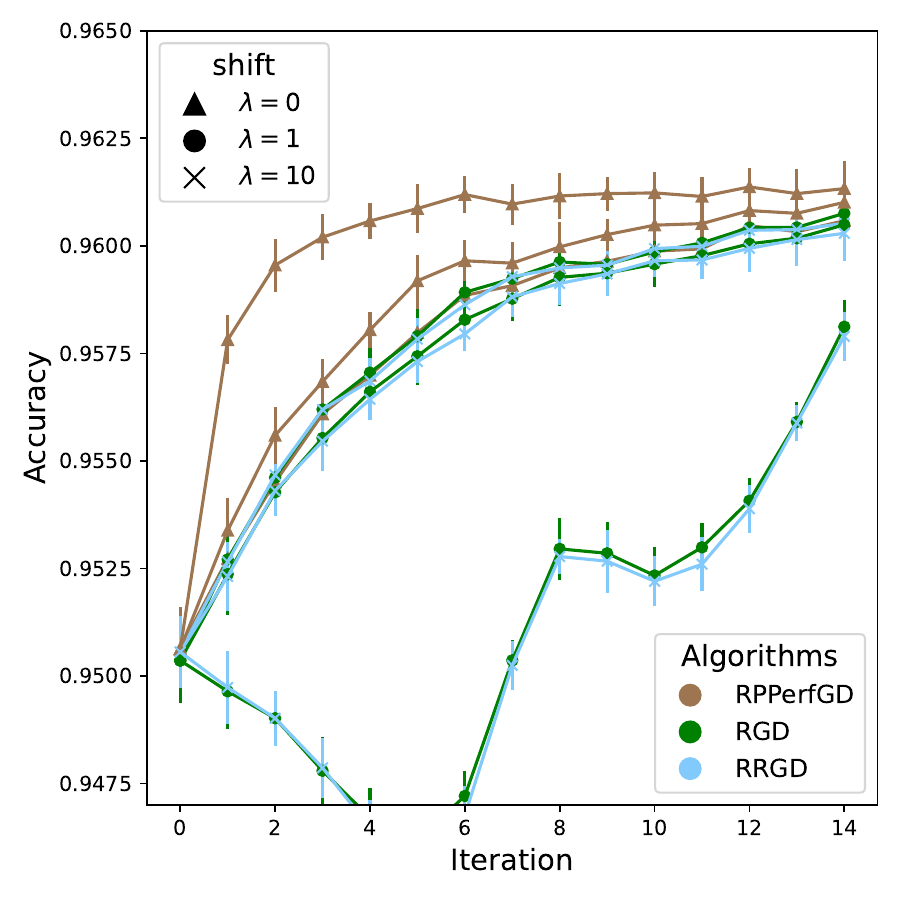}
        \caption{}
        \label{fig:houses}
    \end{subfigure}
    \caption{\textbf{(a)} Logistic regression to classify two Gaussian
    distributions centered in $(0,0)$ and $(-1, -1)$ and different
    magnitudes of performative effects $\gamma$. We report
    the accuracy for three different magnitudes of the performative
    effects, from no performative effect ($\gamma = 0$) to a strong
    one ($\gamma=1$). \textbf{(b)} we report the position of
    the parameter $\theta$ in its 2D-space, starting from $(0,0)$ and
    following different paths depending on the algorithm. \textbf{(c)}
    Accuracy of a classification with quadratic loss on two Gaussian
    distributions of dimension $7$ with various levels of variance
    $\sigma$ of the distributions. \textbf{(d)} Same experiments but using the learnt $\Pi$ for RPPerfGD. \textbf{(e)} In this case, distance between the true matrix $\Pi$ and the estimated version. Note that in RGD and RRGD the estimation of $\Pi$ is not used in the algorithm. \textbf{(f)} Logistic regression for the Housing dataset with various magnitude of performative shift $\lambda$ on the coordinates
    $0$, $4$ and $6$. Accuracy is averaged over $20$
    runs.}
\end{figure}

\paragraph{Influence of the performative effect}
In \cref{fig:log2dtraj}, we illustrate how taking into account the
performative effect allows to mitigate regimes with strong
distribution shift. We generate two Gaussian distributions, one fixed
for class $1$ and one moving with mean $\mu(\theta)^T = -(1,1)
+ \gamma \theta^T \operatorname{diag}(0.1, 0.9)$, where $\gamma$ is
the magnitude parameter effect. We learn a logistic regression and
report the accuracy of the predictions as well as the trajectory of
the parameter in $(\theta_1, \theta_2)$-space. As expected, when there
is no performative effect ($\gamma = 0$), all methods are
equivalent. As soon as there are performative effects, Performative
gradient takes advantage over other methods. RRGD proposes an
interesting tradeoff in terms of performance: agnostic to the
performative effect, it still moderates its value and thus the
magnitude of the performative effect.

\paragraph{Stability of the estimator}
In \cref{fig:sfvsrp}, we illustrate the result
of \cref{ex:gaussian_mean_estimation}, by training a classifier
with the square loss, showing that the score-based estimator used in
SFPerfGD becomes unstable in high dimensions. We use Gaussian
distributions of dimension $7$ with two dimensions subject to
performative effects. We vary the variance of the distributions. When
the scale $\sigma$ is small, the variance of the estimator increases
to the point of making learning impossible with unstable trajectories
of the parameter $\theta$. Even when the scale is small enough to
ensure convergence, RPPerfGD provides faster convergence illustrating
its better scalability for high dimensions.

\paragraph{Estimation of $\Pi$}  We estimate $\Pi$ in \cref{fig:withlearning} and \cref{fig:norm} by running a ridge regression along the successive deployments of the model as described in Algorithm~\ref{algo:pilearning}: the ridge penalty ensures that initially the estimate of $\Pi$ is close to zero making the RPPerfGD updates very similar to those of RGD and it is easy to check that order $d$ deployments are enough to obtain a non void estimate of $\Pi$. While this plug-in approach is not guaranteed to converge from a theoretical standpoint, we observe results that very similar to the case where $\Pi$ is fully known.

\paragraph{Houses price prediction}
To simulative performative effects from a dataset, we follow the
methodology of \cite{perdomo_performative}, by shifting the coordinate
$i$ of a factor $\lambda \theta_i$ if the $i$-th coordinate could be
easily modified, and keeping its real value intact otherwise. We use
the binarized version of the Housing
dataset\footnote{\url{https://www.openml.org/d/823}}, where the
outcome is whether the price is high or not. Assuming that a seller
wants to obtain a high price, the high price is the favored
class. Some characteristics are harder to tamper with such as the
location or the income, whereas other can be slightly adjusted such as
the household and the number of bedrooms (a room could be promoted
bedroom). Coordinates $0$, $4$ and $6$ are thus shifted while other
remains identical. We see that when the magnitude of the shift
increases, RPPerfGD outperforms RGD. In particular, it seems that
RPPerfGD succeeds in converging faster than the non performative
approach.

\begin{algorithm}
    \DontPrintSemicolon
    \SetKwInOut{Input}{Input}
    \SetKwInOut{Output}{Output}
    \Input{Stepsize $\eta$, regularizer $\lambda$, starting $\theta_0$, Loss $\ell$}
    \Output{Parameters $\theta_{K}$ and diagonal matrix $\Pi_{K}$}
    $\Pi_0 \gets 0_{d \times d}$ \tcp{initialize $\Pi$ as a zero matrix of size $d \times d$}
    
    \For{$k \in \{0, \ldots, K-1\}$}{
        Receive $n$ samples $\{x_k^i\}_{i=1}^n \sim D(\theta_k)$ with $n_{0}$ samples of label $-1$ denoted $(x_{0, k})_k$\;
        Compute $\nabla_1 \gets \frac{1}{n} \sum_{i=1}^n \nabla_{\theta} \ell(x_k^i, \theta_k)$\tcp{Non performative part of the gradient}
        Compute $\nabla_2 \gets \frac{1}{n} \Pi_k^{\top} \sum_{i=1}^{n_0} \nabla_{x} \ell(x_{0,k}^{i}, \theta_k)$\tcp{Performative part of the gradient over negative samples}
        $\theta_{k+1} \gets \theta_k - \eta (\nabla_1 + \nabla_2)$ \tcp{Gradient Descent step}
        $\Pi_{k+1} \gets \operatorname{argmin} \sum_{j=1}^k \sum_{l=1}^{n_0} \| x_{0,j}^{l} - \hat{\mu} - \Pi \theta_j \|^2 + \lambda \| \Pi \|^2$  \tcp{$\hat{\mu}$ is the estimated mean of the class}
    }
    \caption{RPPerfGD with $\Pi$ learning}
    \label{algo:pilearning}
    \end{algorithm}

%% file: supplementary.tex
\subsection{Proofs for Example~\ref{ex:gaussian_mean_estimation}}
\label{app:proofvar}
Using the notations of Example~\ref{ex:gaussian_mean_estimation} the score function estimator may we written as $\Pi^T \frac{1}{2\sigma^2}G$, where $G = \|U+a\|^2 U$ with $U \sim
\mathcal{N}(0,\sigma^2 I_d)$ and $a \in \mathbb{R}^d$ is a deterministic vector depending on the parameters $\theta, \theta'$ and the performative effect. To compute the expectation and covariance matrix of $G$ we will use Isserlis' (or Wick's probability) theorem which state that
\begin{enumerate}[label=(\alph*)]
    \item $\E[U_{i_1} \dots U_{i_{2m+1}}] = 0$, for any $\{i_1, \dots, i_{2m+1}\} \in \{1, \dots, d\}^{2m+1}$
    \item $$\E[U_{i_1} \dots U_{i_{2m}}] = \sigma^{2m} \sum_{\{j_1, k_1\}, \dots, \{j_m, k_m\} \in \mathcal{P}(\{i_1, \dots, i_{2m}\})} \delta_{j_1 k_1} \dots \delta_{j_m k_m}$$
    where $\mathcal{P}(\{i_1, \dots, i_{2m}\})$ denotes all the distinct ways of partitioning $\{i_1, \dots, i_{2m}\}$ into non-overlapping (unordered) pairs and $\delta$ is the Kronecker delta. It is easily checked that the number of partitions in $\mathcal{P}(\{i_1, \dots, i_{2m}\})$ is equal to ${2m \choose m}\frac{m!}{2^m}$ which is also equal to the product of all odd numbers between 1 and $2m-1$.
\end{enumerate}

For the expectation, $\E[G] = \E[(\|U\|^2 + \|a\|^2 + 2 a^T U)U] = 2 \E(UU^T) a = 2 \sigma^2 a$, as the expansion of all other terms would involve and odd number of coordinates of $U$.

Let $M = \E[G G^T]$, we have
\begin{multline*}
    M_{ij} = \E \left[ \sum_{k=1}^d (U_k+a_k)^2 \sum_{l=1}^d (U_l+a_l)^2 U_i U_j \right] \\
    = \E \left[ \sum_{k=1}^d \sum_{l=1}^d \left( U_k^2 U_l^2 + U_k^2 a_l^2 + U_l^2 a_k^2 + 4 a_l a_k U_k U_l + a_k^2 a_l^2 \right) U_i U_j \right]
\end{multline*}
omitting terms in the expansion that involve and odd number of coordinates of $U$ (which have 0 expectation). Now, we apply Isserlis' theorem to each term in this decomposition, starting with the lowest order (rightmost) ones:
\begin{align*}
    & \E \left[\sum_{k=1}^d \sum_{l=1}^d a_k^2 a_l^2 U_i U_j\right] = \|a\|^4 \sigma^2 \delta_{ij} \\
    & \E \left[\sum_{k=1}^d \sum_{l=1}^d 4 a_l a_k U_k U_l U_i U_j\right] = 4 \sigma^4 \sum_{k=1}^d \sum_{l=1}^d a_l a_k (\delta_{ij}\delta_{kl} + \delta_{ik}\delta_{jl} + \delta_{il}\delta_{jk}) = 4 \sigma^4 (\|a\|^2 + 2 a_i a_j) \delta_{ij} \\
    & \E \left[\sum_{k=1}^d \sum_{l=1}^d U_k^2 U_i U_j a_l^2\right] = \sigma^4 \|a\|^2 \left(\sum_{k=1}^d (\delta_{ij} + 2 \delta_{ki}\delta_{kj})\right) =  \sigma^4 \|a\|^2 (d+2) \delta_{ij} \\
    & \E \left[\sum_{k=1}^d \sum_{l=1}^d U_k^2 U_l^2 U_i U_j\right] = \sigma^6 \left( \delta_{ij} + 2 \delta_{il}\delta_{jl} + 2 \delta_{ik}\delta_{jk} + 2 \delta_{kl}\delta_{ij} + 8 \delta_{kl}\delta_{ki}\delta_{ij}\right) = \sigma^6 (d^2 + 6 d + 8) \delta_{ij}
\end{align*}
where the last decomposition is obtained by examination of the ${6\choose 3} \frac{3!}{2^3}=15$ possible partitions of $\{k,k,l,l,i,j\}$ in 3 pairs of indices. Putting all together, one obtains
\[
    M = \left((d^2 + 6 d + 8)\sigma^6 + 2 (d+4) \|a\|^2 \sigma^4 + \|a\|^4 \sigma^2 \right) I_d + 8 \sigma^4 a a^T
\]
which yields
\begin{equation}
  \Cov(G) = \left((d^2 + 6 d + 8)\sigma^6 + 2 (d+4) \|a\|^2 \sigma^4 + \|a\|^4 \sigma^2 \right) I_d + 4 \sigma^4 a a^T
  \label{eq:CovG}    
\end{equation}

Subtracting a scalar baseline $m$ yields the estimator $\tilde{G} = (\|U+a\|^2-m)U$ which has the same expectation as $G$. In terms of covariances, one has
\begin{equation*}
  \Cov(\tilde{G}) = \Cov(G) + m^2 I_d -2m \E\left(\|U+a\|^2 U U^T\right)
\end{equation*}
The rightmost expression, when expanded, features terms have already been met in the computation above and it is easy to check that
\[
    \E\left(\|U+a\|^2 U U^T\right) = ((d+2)\sigma^4 + \|a\|^2\sigma^2)I_d
\]
Hence,
\[
    \Cov(\tilde{G}) = \Cov(G) + \left(m^2 -2m((d+2)\sigma^4 + \|a\|^2\sigma^2)\right) I_d
\]
In the above equation, the scalar term $m^2 -2m((d+2)\sigma^4 + \|a\|^2\sigma^2)$ is minimized by choosing $m=(d+2)\sigma^2 + \|a\|^2$ and is equal to $-\left((d+2)\sigma^2+\|a\|^2\right)^2\sigma^2$, which, combined with~\cref{eq:CovG} yields 
\begin{equation}
  \Cov(\tilde{G}) \geq \left(2(d+2)\sigma^6 + 4 \|a\|^2 \sigma^4 \right) I_d + 4 \sigma^4 a a^T
  \label{eq:CovGtilda}  
\end{equation}

\subsection{Proof of Theorem~\ref{thm:convex_performative_classif}}
\label{app:proofconv}
\begin{proof}
  For (a), \cref{eq:DPR_classification_shift} may be rewritten as
  \[
    \operatorname{PR}(\theta) = \rho \E[(U_1^T\theta-1)^2] + (1-\rho)\E[((U_0+\Pi\theta)^T\theta+1)^2]
  \]
  Denoting $\E(U_i) = \mu_i$ and $\operatorname{Cov}_\theta(U_i) = \Sigma_i$, for $i\in\{0,1\}$, one has
  \begin{multline*}
    \operatorname{PR}(\theta) = \rho \E[((U_1-\mu_1)^T\theta-(1-\mu_1^T\theta))^2] + (1-\rho)\E[((U_0-\mu_0)^T\theta+(\mu_0+\Pi\theta)^T\theta+1)^2] \\
    = \rho [\|\theta\|^2_{\Sigma_1} + (1-\mu_1^T\theta)^2] + (1-\rho)[\|\theta\|^2_{\Sigma_0} + ((\mu_0+\Pi\theta)^T\theta+1)^2]
  \end{multline*}
  Both squared norms are convex as well as the squares of,
  respectively, an affine function and a convex second order
  polynomial (as $\Pi$ is positive semidefinite).

For (b), examining
\begin{equation}
  \label{eq:perfomative_risk_Phi}
  \operatorname{PR}(\theta) = \rho \E[\Phi(U_1^T\theta)] + (1-\rho)\E[\Phi(-(U_0+\Pi\theta)^T\theta)]
\end{equation}
one observes that
\begin{itemize}
  \item $\Phi(u_1^T\theta)$ is convex by our assumption on $\Phi$ (for any value of $u_1$);
  \item $(u_0+\Pi\theta)^T\theta$ is a convex second order
    (multivariate) polynomial in $\theta$ when $\Pi$ is positive semidefinite and
    $v\mapsto\Phi(-v)$ is convex non decreasing, hence
    $\Phi(-(u_0+\Pi\theta)^T\theta)$ is also convex.
\end{itemize}
Thus $\operatorname{PR}(\theta)$ is also convex in $\theta$ as the
expectation of convex functions.
\end{proof}

\subsection{Proof of Theorem~\ref{thm:variational_form_performative_risk}}
\label{app:proofrob}
\begin{proof}
To obtain~\cref{eq:adversarial_formulation}, as $v\mapsto\Phi(-v)$ is
non decreasing, one has
\[
  \max_{\{\Delta u_0 : \|\Delta u_0\|_{\Pi^{-1}} \leq \|\theta\|_{\Pi}\}}\Phi\left(-(u_0+\Delta u_0)^T\theta\right) = \Phi\left(-u_0^T\theta- \max_{\{\Delta u_0 : \|\Delta u_0\|_{\Pi^{-1}} \leq \|\theta\|_{\Pi}\}}(\Delta u_0)^T\theta\right)
\]
for any outcome $u_0$ of the random variable $U_0$. The maximization occurs for $\Delta u_0 = \Pi \theta$, which does not depend on $u_0$, leading to $(\Delta u_0)^T\theta = \|\theta\|_\Pi^2$ and thus
\[
  \max_{\{\Delta U_0 : \|\Delta U_0\|_{\Pi^{-1}} \leq \|\theta\|_{\Pi}\}}\Phi\left(-(U_0+\Delta U_0)^T\theta\right) = \Phi\left(-U_0^T\theta - \|\theta\|^2_\Pi \right)
\]
whose expectation is recognized as the second term
of~\cref{eq:perfomative_risk_Phi}.
\end{proof}

\subsection{Proof of Theorem~\ref{thm:regularization}}

\begin{proof}
  Recall from the proof of Theorem~\ref{thm:convex_performative_classif} that the performative risk can be rewritten as follows.
  \begin{align*}
    \operatorname{PR}(\theta) 
    &= \rho \E[\Phi(U_1^T\theta)] + (1-\rho)\E[\Phi(-(U_0+\Pi\theta)^T\theta)]\\
    &= \rho \E[\Phi(U_1^T\theta)] + (1-\rho)\E[\Phi(-U_0^T\theta - \|\theta\|^2_\Pi)]
  \end{align*}
  Let $\mu_\rho = \rho\mu_1 - (1-\rho)\mu_o$. We have,
  \begin{align*}
    \Phi(0) = \operatorname{PR}(0)
    &\geq \operatorname{PR}(\theta^*) \\
    &= \rho \E[\Phi(U_1^T\theta^*)] +  (1-\rho)\E[\Phi(-U_0^T\theta^* - \|\theta^*\|^2_\Pi)]\\
    &\geq  \rho \Phi(\E[U_1^T\theta^*]) + (1-\rho)\Phi(\E[-U_0^T\theta^* - \|\theta^*\|^2_\Pi])\\
    &= \rho \Phi(\mu_1^T\theta^*) + (1-\rho)\Phi(-\mu_0^T\theta^* - \|\theta^*\|^2_\Pi)\\
    &\geq  \Phi \left( \rho\mu_1^T\theta^* - (1-\rho)(\mu_0^T\theta^* + \|\theta^*\|^2_\Pi) \right)\\
    &= \Phi \left( \mu_\rho^T\theta^* - (1-\rho)\|\theta^*\|^2_\Pi \right)
  \end{align*}
  where we have successively used Jensen's inequality and the convexity of $\Phi$. 
  Since $\Phi$ is non-increasing, we must have $0\leq \mu_\rho^T\theta^* - (1-\rho)\|\theta^*\|^2_\Pi$. Denoting $\beta = \Pi^\frac{1}{2}\theta^*$, it holds that
  \begin{align*}
    (1-\rho)\|\beta\|^2
    \leq \mu_\rho^T \Pi^{-\frac{1}{2}}\beta
    \leq \|\Pi^{-\frac{1}{2}}\mu_\rho\| \|\beta \|
  \end{align*}
  where that last inequality is obtained using Cauchy–Schwarz. Hence, $\|\theta^*\|_\Pi = \|\beta \| \leq \frac{\| \mu_\rho^T \Pi^{-\frac{1}{2}}\|}{1-\rho}$.
\end{proof}

%% file: appexpe.tex
\section{Numerical Experiments}
\label{app:expe}

\subsection{Full parameters list}

In this section, we report all the parameters needed to reproduce the figures in the paper. Note that we always use the same step size for all the methods. This choice stems from the fact that the methods are equivalent (up to the regularization parameter) when there is no performative effect.

\begin{table}[h!]
    \caption{Parameters used for figure \cref{fig:log2dtraj}}
    \centering
    \begin{tabular}{@{}ll@{}}
        \toprule
        \textbf{Parameter} & \textbf{Value} \\
        \midrule
        Number of iterations ($\text{num\_iter}$) & 100 \\
        Sample size ($n$) & 1000 \\
        Scale ($\sigma$) & 0.5 \\
        Average number of iterations ($\text{num\_iter\_average}$) & 100 \\
        Step size ($\text{step\_size}$) & 0.1 \\
        Regularization parameter ($\lambda$) & $3 \times 10^{-2}$ \\
        \bottomrule
    \end{tabular}
    \label{tab:fig2a}
\end{table}

\begin{table}[h!]
    \caption{Parameters used for figure \cref{fig:sfvsrp}}
    \label{tab:fig2b}
    \centering
    \begin{tabular}{@{}ll@{}}
        \toprule
        \textbf{Parameter} & \textbf{Value} \\
        \midrule
        Number of iterations ($\text{num\_iter}$) & 25 \\
        Sample size ($n$) & 1000 \\
        Initial scale ($\text{scale}_0$) & 0.5 \\
        Transition probability matrix ($\Pi$) & $\text{diag}([0.1, 3, 0, 0, 0, 0, 0])$ \\
        Mean of class $0$ ($\mu$) & $\begin{bmatrix}1 & 2 & 0.5 & 0.5 & 0 & 0 & 0\end{bmatrix}$ \\
        Average number of iterations ($\text{num\_iter\_average}$) & 100 \\
        Step size ($\text{step\_size}$) & 0.1 \\
        Regularization parameter ($\lambda$) & $10^{-1}$ \\
        \bottomrule
    \end{tabular}
\end{table}

\begin{table}[h!]
    \caption{Parameters used for \cref{fig:houses}}
    \label{tab:fig3}
    \centering
    \begin{tabular}{@{}ll@{}}
        \toprule
        \textbf{Parameter} & \textbf{Value} \\
        \midrule
        Number of iterations ($\text{num\_iter}$) & 15 \\
        Sample size ($n$) & 18000 \\
        Number of runs ($\text{n\_runs}$) & 20 \\
        Step size ($\text{step\_size}$) & 0.2 \\
        Regularization parameter ($\lambda$) & $5 \times 10^{-3}$ \\
        \bottomrule
    \end{tabular}
\end{table}

\subsection{Repeated Risk Minimization}
\label{app:rrm}

In this section, we report the same learning tasks as those reported in the main text, but for Repeated Risk Minimization (RRM). In every setting, as soon as the performative effect is not negligible, the technique diverges. To ensure the readability of the figure and avoid shrinking the differences between the other algorithms, we report it separately.

\begin{figure}[htb]
    \centering
    \begin{subfigure}[b]{0.6\textwidth}
        \centering
        \includegraphics[width=\textwidth]{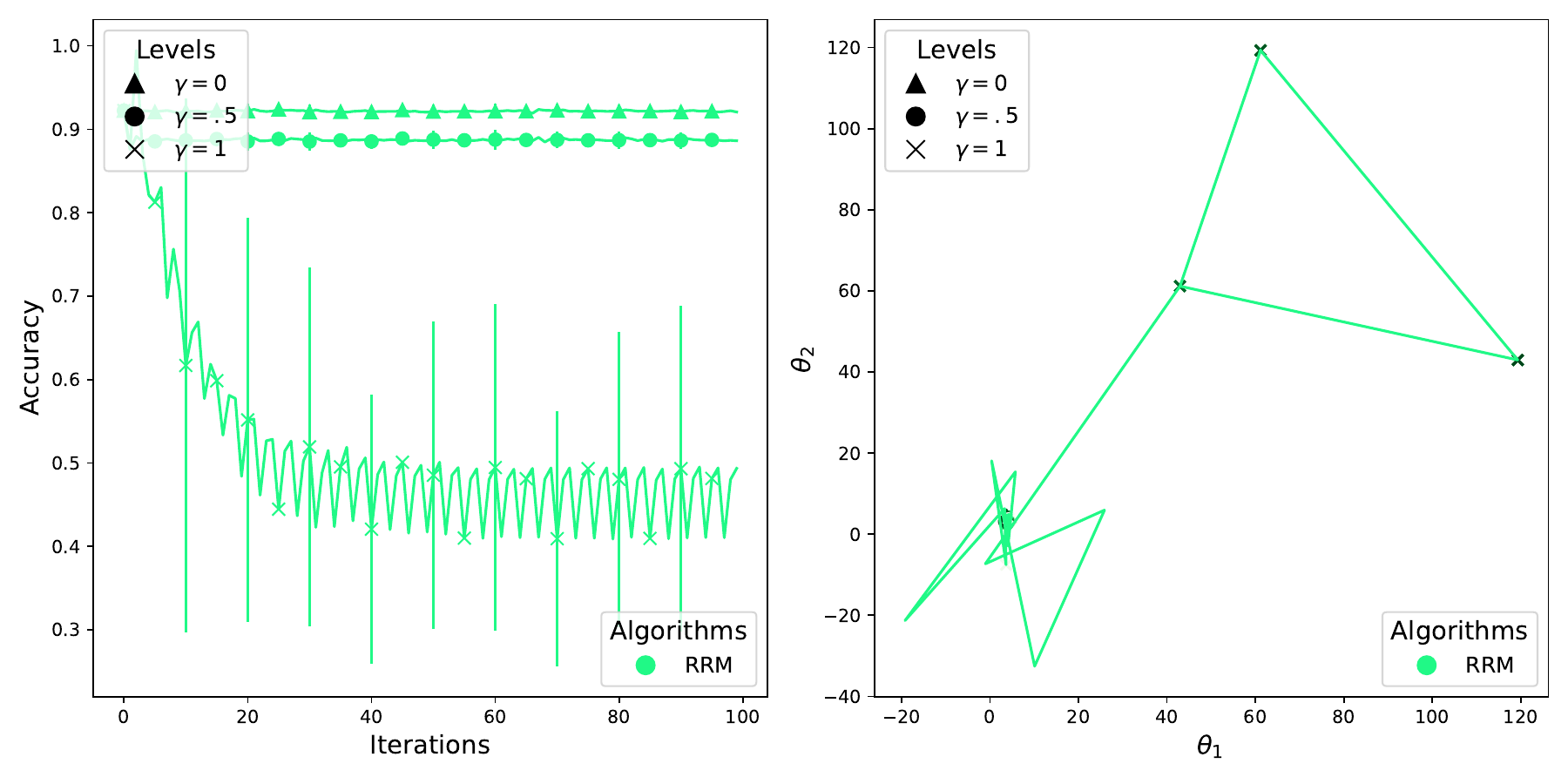}
        \caption{}
        \label{fig:log2trajrrm}
    \end{subfigure}
    \hfill
    \begin{subfigure}[b]{0.3\textwidth}
        \centering
        \includegraphics[width=\textwidth]{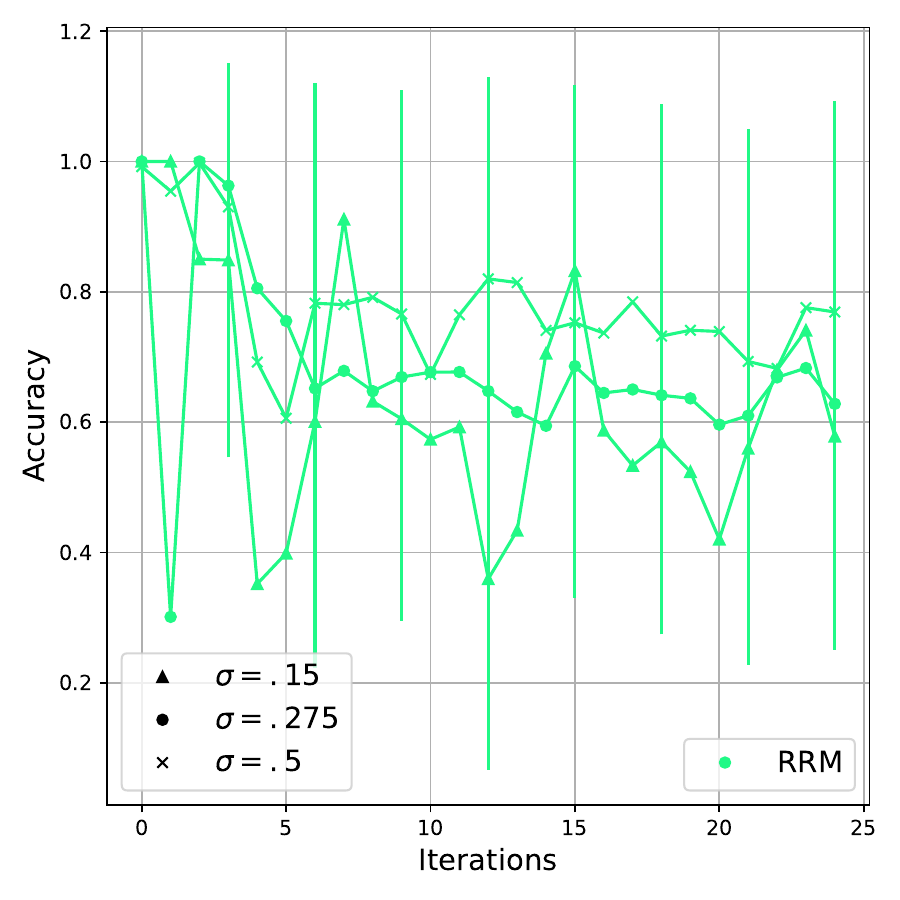}
        \caption{}
        \label{fig:sfvsrprrm}
    \end{subfigure}
    \caption{(a) Learning a logistic regression between two Gaussian distributions centered in $(0,0)$ and $(-1, 1)$ and different magnitude of performative effects $\gamma$. (b) Accuracy of a classification with quadratic loss on two Gaussian of dimension $7$ with various level of noise $\sigma$}
\end{figure}

%% file: checklist.tex
\newpage
\section*{NeurIPS Paper Checklist}

\begin{enumerate}

\item {\bf Claims}
    \item[] Question: Do the main claims made in the abstract and introduction accurately reflect the paper's contributions and scope?
    \item[] Answer: \answerYes{} %
    \item[] Justification: We define our approach with pushforward measure in \cref{sec:pushforward}, the application to classification in\cref{sec:classif}, the new theorems on convexity in \cref{sec:convexity} and the relation with robustness in \cref{sec:robustness}

\item {\bf Limitations}
    \item[] Question: Does the paper discuss the limitations of the work performed by the authors?
    \item[] Answer: \answerYes{} %
    \item[] Justification: We discuss our model as pushforward measures in \cref{sec:pushforward} and we discuss our theorems to ensure convexity in \cref{sec:convexity}. For the experiments, we acknowledge the lack of true performative datasets and use the same simulations approximation as in previous work, as described in \cref{sec:expe}.

\item {\bf Theory Assumptions and Proofs}
    \item[] Question: For each theoretical result, does the paper provide the full set of assumptions and a complete (and correct) proof?
    \item[] Answer: \answerYes{} %
    \item[] Justification: All theorems are defined with their full set of assumptions, with proofs just after them or in the supplementary material (\cref{sec:app}).

\item {\bf Experimental Result Reproducibility}
    \item[] Question: Does the paper fully disclose all the information needed to reproduce the main experimental results of the paper to the extent that it affects the main claims and/or conclusions of the paper (regardless of whether the code and data are provided or not)?
    \item[] Answer: \answerYes{} %
    \item[] Justification: The experiments are described in \cref{sec:expe}. Additional choice of parameters are reported in \cref{sec:app}.

\item {\bf Open access to data and code}
    \item[] Question: Does the paper provide open access to the data and code, with sufficient instructions to faithfully reproduce the main experimental results, as described in supplemental material?
    \item[] Answer: \answerNo{} %
    \item[] Justification: the code will be publicly released after publication

\item {\bf Experimental Setting/Details}
    \item[] Question: Does the paper specify all the training and test details (e.g., data splits, hyperparameters, how they were chosen, type of optimizer, etc.) necessary to understand the results?
    \item[] Answer: \answerYes{} %
    \item[] Justification: The detailed parameters are in Appendix (\cref{app:expe}). The choice of parameters are quite limited as the studied models have very simple architecture 

\item {\bf Experiment Statistical Significance}
    \item[] Question: Does the paper report error bars suitably and correctly defined or other appropriate information about the statistical significance of the experiments?
    \item[] Answer: \answerYes{} %
    \item[] Justification: All figures report standard deviation over the runs, as stated in legend.
\item {\bf Experiments Compute Resources}
    \item[] Question: For each experiment, does the paper provide sufficient information on the computer resources (type of compute workers, memory, time of execution) needed to reproduce the experiments?
    \item[] Answer: \answerNA{} %
    \item[] Justification: The experiments are only using small datasets so all runs were done on a single laptop where the time of execution is very small (few seconds per run).
    
\item {\bf Code Of Ethics}
    \item[] Question: Does the research conducted in the paper conform, in every respect, with the NeurIPS Code of Ethics \url{https://neurips.cc/public/EthicsGuidelines}?
    \item[] Answer: \answerYes{} %
    \item[] Justification: The paper is conform to the NeurIPS code of Ethics.

\item {\bf Broader Impacts}
    \item[] Question: Does the paper discuss both potential positive societal impacts and negative societal impacts of the work performed?
    \item[] Answer: \answerNA{} %
    \item[] Justification: The paper is theoretical and should not have directly any societal impacts. In fact, raising the attention on the feedback loop that might occurs in performative learning could lead to a positive societal impact, if better understood.
\item {\bf Safeguards}
    \item[] Question: Does the paper describe safeguards that have been put in place for responsible release of data or models that have a high risk for misuse (e.g., pretrained language models, image generators, or scraped datasets)?
    \item[] Answer: \answerNA{} %
    \item[] Justification: The paper poses no such risks.

\item {\bf Licenses for existing assets}
    \item[] Question: Are the creators or original owners of assets (e.g., code, data, models), used in the paper, properly credited and are the license and terms of use explicitly mentioned and properly respected?
    \item[] Answer: \answerYes{} %
    \item[] Justification: We only reuse one dataset (Housing) and we cite it.

\item {\bf New Assets}
    \item[] Question: Are new assets introduced in the paper well documented and is the documentation provided alongside the assets?
    \item[] Answer: \answerNA{} %
    \item[] Justification:  The paper does not release new assets.

\item {\bf Crowdsourcing and Research with Human Subjects}
    \item[] Question: For crowdsourcing experiments and research with human subjects, does the paper include the full text of instructions given to participants and screenshots, if applicable, as well as details about compensation (if any)? 
    \item[] Answer: \answerNA{} %
    \item[] Justification: The paper does not involve crowdsourcing nor research with human subjects.

\item {\bf Institutional Review Board (IRB) Approvals or Equivalent for Research with Human Subjects}
    \item[] Question: Does the paper describe potential risks incurred by study participants, whether such risks were disclosed to the subjects, and whether Institutional Review Board (IRB) approvals (or an equivalent approval/review based on the requirements of your country or institution) were obtained?
    \item[] Answer: \answerNA{} %
    \item[] Justification: The paper does not involve crowdsourcing nor research with human subjects.

\end{enumerate}